\documentclass[conference,letterpaper, onecolumn]{IEEEtran}
\IEEEoverridecommandlockouts

\usepackage[letterpaper,margin=1in,footskip=0.25in]{geometry}
\usepackage{cite}
\usepackage{changes}
\usepackage{amsmath,amssymb,amsfonts,amsthm}
\usepackage{algorithm, algorithmic}
\usepackage{graphicx}
\usepackage{textcomp}
\usepackage{booktabs}
\usepackage{multirow}
\usepackage{xcolor}
\theoremstyle{definition}

\theoremstyle{remark}

\newtheorem{theorem}{Theorem}

\def\BibTeX{{\rm B\kern-.05em{\sc i\kern-.025em b}\kern-.08em
    T\kern-.1667em\lower.7ex\hbox{E}\kern-.125emX}}

\definechangesauthor[name=Longsheng, color=blue]{LJ} 
\DeclareMathOperator*{\argmax}{argmax}

\begin{document}

\title{Respect Your Emotion: Human-Multi-Robot Teaming based on  Regret Decision Model\\
\thanks{The authors are with the Mechanical Engineering Department, Clemson University, Clemson, SC,  29634, USA. E-mail: \{longshj, yue6\}@g.clemson.edu. Research is supported in part by the National Science Foundation under Grant No. CMMI-1454139.}
}

\author{Longsheng~Jiang and Yue~Wang
}

\maketitle

\begin{abstract}
Often, when modeling human decision-making behaviors in the context of human-robot teaming, the emotion aspect of human is ignored. 
Nevertheless, the influence of emotion, in some cases, is not only undeniable but beneficial. 
This work studies the human-like characteristics brought by regret emotion in one-human-multi-robot teaming for the application of domain search. 
In such application, the task management load is outsourced to the robots to reduce the human's workload, freeing the human to do more important work. 
The regret decision model is first used by each robot for deciding whether to request human service, then is extended for optimally queuing the requests from multiple robots.  
For the movement of the robots in the domain search, we designed a path planning algorithm based on dynamic programming for each robot. 
The simulation shows that the human-like characteristics, namely, risk-seeking and risk-aversion, indeed bring some appealing effects for balancing the workload and performance in the human-multi-robot team.  
\end{abstract}

\section{Introduction}
A team of many low-cost small robots may often perform better than one big expensive robot. The robot team is flexible, robust, and the robots are expendable. 
However, low-cost robots do not have highly accurate sensors. 
These drawbacks can be alleviated by teaming the robots with human operators.
Due to this reason, the research in studying the interaction between one human and multiple robots has attracted much interests \cite{kolling2016human}.

One challenge in such collaboration is the balance between human workload and team performance \cite{ chen2014human}.
Humans are much better in sensing and recognition than low-cost robots.
Thus, utilizing humans can attain better task outcomes. 
However, humans are also slower, more expensive and subject to fatigue. 
Hence, human workload must be taken into consideration when designing the interaction between a human and multiple robots. 
For instance, the workload of task management can be outsourced to robots to certain extents \cite{parasuraman2000model}; an intelligent agent can be used to coordinate the group of robots \cite{chen2012supervisory}.

Because the human in the team is not only a collaborator but also a supervisor, he/she holds the responsibility of managing the robots. 
When the management load is delegated to the robots, the robots should make decisions in a way similar to that of the human. 
Thus, a quantitative model of the human's decision-making is needed.

In early days, when needing a quantitative decision-making model, humans were often approximated as optimal controllers \cite{kleinman1970optimal} or expected utility maximizers \cite{keeney1993decisions}. 
However, researchers gradually noticed the inevitable influence of emotion to decision-making \cite{marsella2009ema} and started to incorporate the effects of emotion when designing their robots \cite{velasquez1998robots}. 
Emotion is not only of affection importance, sometimes it actually helps people to make wiser decisions. 
For instance, it is the wariness of future disasters and the regret when upon the disasters that bias people toward purchasing insurances~\cite{loomes1982regret}.

Regret is important for human decision-making. 
It explains people's risk-seeking and risk-averse behaviors. 
Risk-averse happens, as in the above insurance example, when facing two options 1) of  a high cost but low occurrence probability and 2) of a small but sure cost, respectively. 
Risk-seeking happens when facing two options 1) of a moderate cost but high occurrence probability and 2) of a sure cost close to the moderate cost, respectively. 
These two types of decision-making pattern should not be considered as anomalies because they are observed in the majority of people \cite{kahneman1979prospect}. 
 
In this work, we will include the human risk-attitudes in the coordination of a human-multi-robot team in performing a domain search task. 
In the task, the human operator outsources the task management load to each robot and an intelligent agent. 
The novelty is that we use the extended versions of regret theory \cite{loomes1982regret} as the decision-making model, for the robots' decision-making and for the agent's queue-ordering.  
The intuition behind regret theory is that both the risk-seeking and risk-averse behaviors are caused by regretful emotion:
Choosing one option means forgoing the other, and you might regret that you did not choose otherwise. 
The theory is backed by, in neurology, the identification of brain regions which generates regret during decision-making \cite{camille2004involvement}, and by the quantitative measurement of regret theory~\cite{liao2017quantitative}.

We formulate the task search problem in section \ref{sec:probform}, and explain the decision-making of each robot in section \ref{sec:decision-making}.
Section \ref{sec:queue} describes how the several waiting robots are queued optimally.
Section \ref{sec:pathplanning} shows the path planning of each robot. 
We show the simulation and results in section \ref{sec:simulation_result} and conclude the work in section \ref{sec:conclusion}.

\section{Problem formulation} \label{sec:probform}
A group of target objects, with a known total amount, are scattered in a large 2-D domain with unknown locations. 
A team---comprising of multiple robots and a human operator---is tasked to find those objects. 
While the robots are deployed on field, the operator stays in the operating station. 
Robots are equipped with vision sensors, GPS and communication units.
They have direct communication with the station. 

The large domain is divided into small regions with only one robot in each. 
Each region is divided further into cells. 
The size of each cell is assumed to be large enough to contain at most one object. 
Each robot $r \in \mathcal{R}$, where $\mathcal{R}$ is the set of all robots, sweeps the cells one by one to detect if a cell contains an object. 
When a robot $r$ is at cell $x_r$, due to the limitation of its sensing capability, its detection result $Y_r$ might possibly be correct or wrong. 
Alternatively, the robot has the option to request the human to tele-operate itself; the communication channel is established for the human to see through the robot's video camera.
The human detection accuracy is superior, yet with a higher cost. 
The robot needs to choose between accepting its own detection results (option R) and requesting human service (option H). 
When multiple robots choose option H simultaneously, the station queues the requests.

The complete state of robot $r$, $(x_r,Y_r)$, describes both the cell at which the robot currently locates ($x_r$) and correctness of the current detection ($Y_r$). 
It is assumed that the robot motion control is precise enough so that $x_r$ is deterministic. 
As mentioned above, however, the detection result $Y_r$ is uncertain over two possible states: $y_c$ for correct detection and $y_w$ for wrong detection.

The prior probability distribution (belief) of $Y_r$ roots in the prior knowledge of the possible locations of the objects as well as the sensor model of robot $r$. 
The prior probability of an object present (or absent) at cell $x_r$ can be represented as $p(S_x=s_p)$ (or $p(S_x=s_a)$), where $s_p$ (or $s_a$) denotes object presence (or absence). 
For brevity, when denoting probabilities, the random variables in the parentheses are omitted, e.g., $p(S_x=s_p)=p(s_p)$. 
The observations are represented by $O=o_p$ for presence and $O=o_a$ for absence.
The sensor model for robot $r$ is described by four known probabilities: $p(o_a|s_a,r)$, $p(o_p|s_a,r)$, $p(o_p|s_p,r)$ and $p(o_a|s_p,r)$.
The probability $p_r$ for robot $r$ obtaining $Y_r=y_c$ at a cell thus is 
\begin{equation*}
p_r \triangleq p(y_c|r) = p(o_a|s_a,r)p(s_a)+p(o_p|s_p,r)p(s_p).
\label{eqn:prob_x}
\end{equation*}
The probability for robot $r$ to obtain $Y_r=y_w$ at a cell follows immediately: $p(y_w|r)= 1-p_r$. 

The variables in $(x_r,Y_r)$ have different observability as well. We assume that the localization uncertainty in GPS is negligible: $x_r$ is fully observable.
However, $Y_r$ is only partially observable.
Even after an observation at cell $x_r$, the actual state of $Y_r$ is still not known for sure, although the observation helps update the belief. 

In determining whether to request human service, two types of information should be considered: performance and cost. 
Option R has compromised performance. 
Its two possible outcomes, right or wrong, can be defined by the resulting costs. 
We consider the correct detection causing no cost, thus $c(y_c)=0$; the wrong detection however has a cost $c(y_w)<0$. 
The human service is considered to have perfect performance (always $y_c$), but it has an operational cost $c_H<0$. 
The two options are represented in Table~\ref{tb:twooption}. 
When $c_H<c(y_w)<0$, option R is certainly the choice.
However, when $c(y_w)<c_H<0$, how to balance the performance and cost becomes challenging.
The later one is the case we study. 
\begin{table}[!ht]
\renewcommand{\arraystretch}{1.3}
\caption{Two options represented in the original form}
\label{tb:twooption}
\centering
\begin{tabular}{c c c c c c c}
\multicolumn{3}{c}{\bf  Option H} & & \multicolumn{3}{c}{\bf  Option R}\\
\cmidrule{1-3} \cmidrule{5-7}
Cost:   &  \multicolumn{2}{c} { $c_H$} &  & Cost: &   $ 0$    &     $c(y_w)$  \\
Probability: & \multicolumn{2}{c} {$100\%$} &  & Probability: & $p_r$ &      $1-p_r$ \\
\cmidrule{1-3} \cmidrule{5-7}
\end{tabular}
\end{table}

The most popular, yet most straightforward, decision-making method is expected value theory. 
It calculates the net advantage of option H relative to option R as
\begin{equation}
	e_v = c_H - (1-p_r)\,c(y_w). 
\end{equation}
It chooses option H, if $e_v>0$; option R, if $e_v<0$; equally liking, otherwise. 

Multiple robots may request human service simultaneously.
When they send their requests to the intelligent agent at the station, the agent queues them into a waiting line. 
The queuing is urgency based. 
The human serves the robots one by one. 
For the team to accomplish the job, each robot needs to move in the domain. 
Thus, a method of optimal path planning is needed.

\section{Regret Theory Based Decision-making}\label{sec:decision-making}

The name of regret theory comes from its effort to account for the influence of anticipatory regretful emotion, which arises from the comparisons of costs. 
These comparisons can be highlighted by representing Table~\ref{tb:twooption} in the new form in Table~\ref{tb:twooption_comparison}.
The last two columns show the comparisons of costs, with the joint probabilities for the comparisons to occur (i.e., $p_r=1\cdot p_r$ and $1-p_r=1\cdot (1-p_r)$). 
Denote the costs of option R as vector $\mathbf{c}_R \triangleq [0, c(y_w)]$; of option H, as $\mathbf{c}_H \triangleq [c_H, c_H]$. 
We can use $\mathbf{c}_R$ and $\mathbf{c}_H$ as handles to the two options. 
\begin{table}[!th]
\renewcommand{\arraystretch}{1.3}
\caption{Two options represented in the comparative form}
\label{tb:twooption_comparison}
\centering
\begin{tabular}{c | c| c| c}
\hline
\multicolumn{2}{c|}{Joint probability} & $p_r$ & $1-p_r$ \\ \hline
\multirow{2}{*}{Cost}  & {\bf Option H:} $\mathbf{c}_H$ & $c_H$ & $c_H$ \\ \cline{2-4}

                        & {\bf Option R:} $\mathbf{c}_R$ & $0$ & $c(y_w)$ \\ \hline 
\end{tabular}
\end{table}

Mathematically, the regret influence is modelled by a function $Q(\Delta c)$, where the argument $\Delta c$ is the difference of costs in comparison. The function $Q(\Delta c)$ has some important properties \cite{loomes1982regret} that include:
\begin{enumerate}
\item[i)] $Q(\Delta c)$ is an odd function,
\item[ii)] $Q(\Delta c)$ is monotonically increasing.
\end{enumerate}

It is also shown that, in human decision-making, probabilities are subjectively biased \cite{kahneman1979prospect}, and
the evaluation of absolute costs is affected by the range of costs~\cite{kontek2017range}. 
We thus introduce probability weighting functions $w_k(p_r), k=1,2, $~\cite{quiggin1982theory} and a constant cost normalizer $c_\text{range}>0$ to the original regret theory \cite{jiang2019human}. 
The costs after normalization are in $[-1,0]$. 

The decision-making in regret theory is determined by the net advantage of one option with respect to the other option. 
In the case of Table \ref{tb:twooption_comparison}, the net advantage of $\mathbf{c}_H$ when using $\mathbf{c}_R$ as the reference is
\begin{equation}
    e_r(\mathbf{c}_H,\mathbf{c}_R) \triangleq \sum_{k=1}^2 w_k(p_r)\,Q\big(\frac{\mathbf{c}_H(k)-\mathbf{c}_R(k)}{c_\text{range}}\big),
    \label{eqn:regret_theory}
\end{equation}
where $\mathbf{c}_{(\cdot)}(k)$ is the $k$-th element in vector $\mathbf{c}_{(\cdot)}$. 
The two functions $w_1$ and $w_2$ are different but depend on a unique $w$-function such that $w_1(p_r)=w(p_r)$ and $w_2(p_r)=1-w(p_r)$, according to \cite{quiggin1982theory}.
The choice is option H---denoted as $\mathbf{c}_H \succ \mathbf{c}_R$---if $e_r(\mathbf{c}_H,\mathbf{c}_R)>0$; option R ($\mathbf{c}_H \prec \mathbf{c}_R$) if $e_r(\mathbf{c}_H,\mathbf{c}_R)<0$; equally liking ($\mathbf{c}_H \sim \mathbf{c}_R$), otherwise. 
There, the $Q$-function 
and the $w$-function are individual specific and can be measured \cite{jiang2019human}.


In this work, we use the  functional forms for the the $Q$-function and the $w$-function. 
These analytical expressions agree with the empirical data in the literature. 
The $Q$-function \cite{liao2017quantitative} and the $w$-function \cite{prelec1998probability} are defined as,
\begin{IEEEeqnarray}{rcl}
	Q(\Delta	 c) & \triangleq  & \alpha_1 \, \sinh(\alpha_2 \, \Delta c) + \alpha_3\,\Delta c,\label{eqn:Q-function}\\
 w(p_r) & \triangleq & \exp \big(-\beta_1\big(-\log(p_r)\big)^{\beta_2}\big), \label{eqn:w-function}
\end{IEEEeqnarray}
where $\alpha_1,\;\alpha_2,\;\alpha_3>0$ and $\beta_1,\;\beta_2>0$ are parameters specific to individuals and can be estimated with data.   

\section{Regret theory Based Queue-ordering}\label{sec:queue}
It is often the case that multiple robots choose option H simultaneously. 
However, the operator can only serve one robot at a time; the robots need to form a waiting line. 
An ordered waiting line can be represented by a permutation.
Suppose $N$ robots are choosing option H. 
The number of permutations is roughly about $N!$ (some service requests may be rejected).
The $N$ robots form a set $\mathcal{R}_H$. 
Regarding robot $r \in \mathcal{R}_H$, the intelligent agent needs to decide either to reject its request or to designate to it a line position $n \in \{1,2,\ldots,M\}$, and $M \leq N$ because of possible service rejections. 
Thus, for the intelligent agent, there are $M+1$ options for robot $r$, which are defined in set $\mathcal{C}_r\triangleq \{ \mathbf{c}_{H_1}, \ldots, \mathbf{c}_{H_M},\mathbf{c}_{R} |\,p_r \}$, as in Table \ref{tb:manyoption_comparison}. 
\begin{table}[t!]
\renewcommand{\arraystretch}{1.3}
\caption{Multiple options in Comparative form for Robot $r$}
\label{tb:manyoption_comparison}
\centering
\begin{tabular}{c | l| c| c}
\hline
\multicolumn{2}{c|}{Joint probability} & $p_r$ & $1-p_r$ \\ \hline
\multirow{4}{*}{Cost} & {\bf Option H$_1$:} $\mathbf{c}_{H_1}$ & $c_{H_1}$ & $c_{H_1}$ \\ \cline{2-4}
                      & \hspace{0.5cm}$\vdots$       & $\vdots$ & $\vdots$ \\ \cline{2-4}
                      & {\bf Option H$_M$:} $\mathbf{c}_{H_M}$ & $c_{H_M}$ & $c_{H_M}$ \\ \cline{2-4}
                      & {\bf Option R:} $\mathbf{c}_{R}$ & $0$ & $c(y_w)$ \\ \hline 
\end{tabular}
\end{table}
The cost part $\{  \mathbf{c}_{H_1}, \ldots, \mathbf{c}_{H_M}, \mathbf{c}_{R}\}$ is independent to robot $r$, because cost vectors $\mathbf{c}_{H_n}$ depend on line position $n$, and $\mathbf{c}_R$ on cost $c(y_w)$ only; set $\mathcal{C}_r$ relates to robot $r$ only because of probability $p_r$.
As robot $r$ is further down the line, the waiting time is longer, cost $c_{H_n}$ is more negative, i.e., $c_{H_1}>c_{H_2}>\ldots>c_{H_M}$. 
Then, we have the following result. 

\begin{theorem}\label{lem:1certain1uncertain_monotonicity}
   For one certain option $\mathbf{c}_{H_n}$ and one uncertain option $\mathbf{c}_{R}$, the net advantage $e_r(\mathbf{c}_{H_n}, \mathbf{c}_{R})$ is monotonically increasing with respect to scalar cost $c_{H_n}$.
\end{theorem}
\begin{proof}
    Substitute $\mathbf{c}_{H_n}=[c_{H_n},c_{H_n}]$ for $\mathbf{c}_H$ in Eqn. (\ref{eqn:regret_theory}) and take derivative with respect to $c_{H_n}$, we have 
    \begin{equation*}
        \frac{\partial e_r}{\partial c_{H_n}}(\mathbf{c}_{H_n}, \mathbf{c}_{R}) =\frac{1}{c_\text{range}}\sum_{k=1}^2 w_k(p_r)\,Q'\big(\frac{c_{H_n}-\mathbf{c}_R(k)}{c_\text{range}}\big).
    \end{equation*}
   Because of $Q'(\cdot)>0$ from property (ii), the weighting function $w_k(p_r)>0$, $k=1,2$, for any $p_r \neq 0$, and $c_\text{range}>0$, the derivative $\frac{\partial e_r}{\partial c_{H_n}}>0$ for any $c_{H_n}$. 
\end{proof}

Moreover, in an option set with more than two options, what matters is not only which option is superior, but how much one option is better than another.  
To enable the comparison in magnitudes, a common reference option should be chosen against which the net advantage of each option in the set is computed. 
We choose $\mathbf{c}_R$ in $\mathcal{C}_r$ as such reference because option $\mathbf{c}_R$ is the ``bottom line'' of $\mathcal{C}_r$; robot $r$ opts for option H---we call $\mathbf{c}_{H_n}$, $n=1,\ldots, M$, all as option H---if and only if $\mathbf{c}_{H_n} \substack{ \succ \\ \sim} \mathbf{c}_{R}$.

 
We use an augmented permutation matrix $\mathbf{P}$ for recording the order of the $N$ robots in a line of length $M$; whenever a robot is rejected of service (choosing $\mathbf{c}_R$), its position is set as $M+1$.  
The accumulated net advantage function of $\mathbf{P}$ can be defined as
\begin{equation}
	G(\mathbf{P}) \triangleq \sum_{r\in \mathcal{R}_H}e_r(\mathbf{c}_r,\mathbf{c}_{R}),
	\label{eqn:value_permutation}
\end{equation}
where $\mathbf{c}_r \in \mathcal{C}_r$ and $c_R$ is the option R for robot $r$. 
It has the constraint that if $\mathbf{c}_{r_1}=\mathbf{c}_{H_m}$ and $\mathbf{c}_{r_2}=\mathbf{c}_{H_n}$ for any two robots in $\mathcal{R}_H$, respectively, then $m \neq n$. 

We then can find the optimal $\mathbf{P}^*$ in the space $\mathcal{P}$ of $\mathbf{P}$,
\begin{equation}
	\mathbf{P}^* = \argmax_{\mathbf{P} \in \mathcal{P}}\, G(\mathbf{P}). 
	\label{eqn:max_permutation}
\end{equation}
Each of the optimal queues is a sequence such that the robots in urgency are served first, and the urgency depends on how fast their net advantages decrease.  
To find the exact optima, it involves enumerating all $\mathbf{P}$ in $\mathcal{P}$. 
When $N$ grows large, $|\mathcal{P}|\approx N!$, making it difficult to compute $\mathbf{P}^*$ in real-time. 
Hence, we use a fast heuristic algorithm to approximate the optimal solution.  

Recall $c_{H_n}$ denotes the cost of waiting at line position $n$.
Forming a line consists of selecting a robot for each position $n=M, M-1, \ldots, 1$.   
Algorithm~\ref{algm:optimiztion_permutation} describes a simple mechanism for such a selection. 
\begin{algorithm}[H]
\caption{Optimization for the Approximated Model}\label{algm:optimiztion_permutation}
\begin{algorithmic}[1]
\renewcommand{\algorithmicrequire}{\textbf{Input:}}
\renewcommand{\algorithmicensure}{\textbf{Output:}}
\REQUIRE $\mathcal{R}_H$
\ENSURE $\mathbf{P}$
\STATE{Initialization: $M \leftarrow 0$, $\hat{\mathcal{R}}_H \leftarrow \mathcal{R}_H$, and $N_q \leftarrow |\mathcal{R}_H|$}
\WHILE{$M<N_q$}
	\STATE{$M \leftarrow M+1$ and $N_q \leftarrow 0$}
	\FOR{$r \in \hat{\mathcal{R}}_H$}
		\IF{$e_r(\mathbf{c}_{H_M}, \mathbf{c}_{R}) > 0$}
		\STATE{$N_q \leftarrow N_q + 1$}
		\ENDIF
	\ENDFOR
\ENDWHILE
\FOR{$r \in \hat{\mathcal{R}}_H$}
	\IF{$e_r(\mathbf{c}_{H_M}, \mathbf{c}_{R}) \leq 0$}
		\STATE{Reject service to robot $r$.}
		\STATE{Save this choice for robot $r$ to $\mathbf{P}$}
		\STATE{Pruning: $\hat{\mathcal{R}}_H \leftarrow \hat{\mathcal{R}}_H \setminus r$}
	\ENDIF
\ENDFOR
\FOR{$n=M$ counting down to $n=1$}
	\STATE{Select a robot $r$ at position $n$ according to Eqn.~(\ref{eqn:approximate_optimization})}
	\STATE{Save the determined position $n$ of robot $r$ to $\mathbf{P}$}
	\STATE{Pruning: $\hat{\mathcal{R}}_H \leftarrow \hat{\mathcal{R}}_H \setminus r$}
\ENDFOR
\end{algorithmic}
\end{algorithm}

In Algorithm \ref{algm:optimiztion_permutation}, lines 2--9 determine the length $M$ of the queue such that $e_r(\mathbf{c}_{H_n}, \mathbf{c}_{R}) > 0$, $n \in \{1,2,\ldots,M\}$ for the robots. 
If, because of waiting, the net advantage becomes $e_r(\mathbf{c}_{H_n}, \mathbf{c}_R)\leq 0$ for some $r$, then the robot should opt for option R, hence $M \leq N$. 
Lines 10--15 reject serving the robots which are determined better to choose option R. 
Lines 16--20 proceed to select a robot $r$ at position $n$ from $n=M$ to $n=1$, using  
\begin{equation}
 r^* = \argmax_{r\in \hat{\mathcal{R}}_H}\, \Delta e_r(\mathbf{c}_{H_{n}}, \mathbf{c}_{H_{n+1}}),
	\label{eqn:approximate_optimization}
\end{equation}
and $\Delta e_r(\mathbf{c}_{H_{n}}, \mathbf{c}_{H_{n+1}}) \triangleq e_r(\mathbf{c}_{H_{n+1}}, \mathbf{c}_R)-e_r(\mathbf{c}_{H_{n}}, \mathbf{c}_R)$.
Since in the line, $\mathbf{c}_{H_{n}}$ only exists for $n\in \{1,\ldots,M\}$, the artificial $\mathbf{c}_{H_{M+1}}$ is defined as $\mathbf{c}_{H_{M+1}} \triangleq \mathbf{c}_{R}$.   
It is important that this algorithm works backward.
Due to Theorem \ref{lem:1certain1uncertain_monotonicity}, $\Delta e_r(\mathbf{c}_{H_{n}}, \mathbf{c}_{H_{n+1}})<0$.
What Eqn. (\ref{eqn:approximate_optimization}) does is to find the $\Delta e_r$ which is closest to zero, meaning the change is not rapid and the robot is less urgent.
We thus put this robot toward the end of the waiting line.

\section{Path Planning}\label{sec:pathplanning}
This section focuses on the path planning of one robot, thus notation $r$ is omitted without loss of specificity.

In this work, we employ two strategies for the searching.
First, we use a sweeping strategy: 
Each cell is visited only once, and the search ends as soon as all the objects are found. 
This strategy is efficient: It avoids repetitive and complete coverage of the domain.
Its drawback, however, is that there is no second chance for each detection. 
Second, we adopt modular design for the decision-making and the path planning.
Its virtue is allowing us to focus on developing and improving the two modules independently. 
The price is the disregard for the vexing influence between the two modules, which we will address in our future work.
 

The state of the robot is $(x,Y)$. 
We model the state transition of the robot as in \figurename~\ref{fig:MOMDP}. 
\begin{figure}[!b]
\centering
\includegraphics[width=8cm]{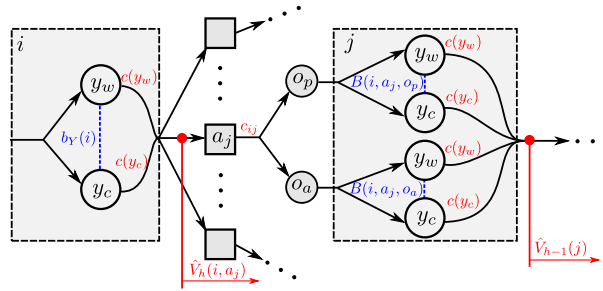}    
\caption{ The state transition of the robot moving from cell $i$ to $j$.}  
\label{fig:MOMDP}             
\end{figure}
At $x=i$, the robot is uncertain if its detection state is $y_w$ or $y_c$. 
The uncertainty is denoted with a belief $b_Y(i)$. 
The robot acts by choosing the next cell among all unvisited cells. 
With a candidate action $a_j$, the robot will move to $x=j$.
It will receive an observation, $o_p$ or $o_a$, at cell $j$.
Based on the observation, the robot will update its belief of detection states at cell $j$, denoted as $B(i, a_j, o_p)$ or $B(i,a_j,o_a)$, respectively. 
Then, the robot will further select another next cell. 
Depending on the detection state the robot is in and the action it selects, different costs apply, denoted as $c(y_w)$, $c(y_c)$ and $c_{ij}$.
Cumulative costs $\hat{V}_h$ are accumulated from the indicated locations to the end of the planning horizon, where $h$ is the length to end of the horizon.
We assume that for any two cells $i$ and $j$, belief $b_Y(i)$ and $b_Y(j)$ are independent. 
The optimal policy at cell $i$ can be obtained through dynamical programming:
\begin{equation*}
    \hat{V}_h(i)=\max_{a_j \in \mathcal{A}}\Big[\Big(c_{ij}+\sum_{y_j\in \mathcal{Y}}p(y_j)\,c(y_j)\Big)+\hat{V}_{h-1}(j)\Big],
    \label{eqn:thm1}
\end{equation*}
where set $\mathcal{A}$ contains the unvisited cells, $\mathcal{Y} \triangleq \{y_w, y_c\}$.

\section{Simulation and Results}\label{sec:simulation_result}
We simulated a team consisting of 1 human and 10 robots, as in \figurename~\ref{fig:problem_formulation}. 
Each robot was assigned to a region with 10 by 10 cells. 
In total there were 1000 cells. 
Each cell had a prior probability of containing an object that was drawn randomly from interval $[0, 0.2]$. 
In total there were $100$ cells containing objects. 

For the regret decision model, we used the parameters of subject 12 in our previous work~\cite{jiang2019human}.  
\begin{figure}[!th]
\centering
\includegraphics[width=7.5cm]{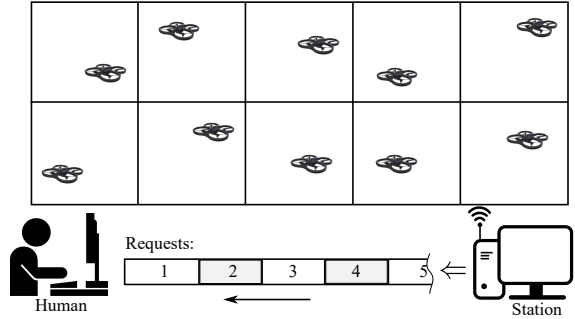}    
\caption{The collaboration of a human and multi-robot search team.}  
\label{fig:problem_formulation}                                                          
\end{figure}

We simulated the search task in two conditions.
In the condition of high sensor accuracy and low human cost, we set the sensor model of each robot as $p(o_a|s_a)=p(o_p|s_p)=0.7$. 
These sensors were moderate because they were on low-cost robots.  
The observations in cells were simulated once randomly using the prior distribution and the sensor model, and then saved for the following simulations. 
Thus, we had the same simulation environment to compare regret theory and expected value theory. 
The cost of being wrong (a miss or a false alarm) was $c(y_w)=-30$. 
The cost of being correct (a hit or a correct rejection) with human service was relatively low, $c_{H_n}=-1.5\,n$, where $n$ was the line position assigned by the intelligent agent.  
As a contrast, in the condition of low sensor accuracy and high human cost, we set the sensor model as $p(o_a|s_a)=p(o_p|s_p)=0.5$ to simulate severe situations.
In this case, the wrong detection cost was still $c(y_w)=-30$ but the human service cost was set at $c_{H_n}=-10\,n$.

\figurename~\ref{fig:WaitingLine} shows the difference between queue-ordering with the regret decision model and the expected value decision model under the high sensor accuracy low human cost condition.   
As seen, the waiting line based on regret theory is longer than based on expected value theory. 
It indicates that more robots decided to request human service under the regret decision model, comparing with expected value theory. 
In this scenario, the robots with the regret decision model became risk-averse. 
The intelligent agent was also risk-averse because it retained a longer line. 
Also, the queue-ordering was dynamic.
At time $t$ there were 7 robots waiting under the regret decision model. 
In the simulation we set that the human operator could process 3 requests each time step. 
Thus, at time $t+1$, the length of the waiting line reduced to 4. 
At time $t+2$, however, there were new requests from the robots.
These requests did not necessarily queue at the end of the line. 
Depending on their urgency, they could be inserted to the front; In \figurename~\ref{fig:WaitingLine}, they were inserted before the request from robot 1. 
\begin{figure}[!th]
\centering
\includegraphics[width=8cm]{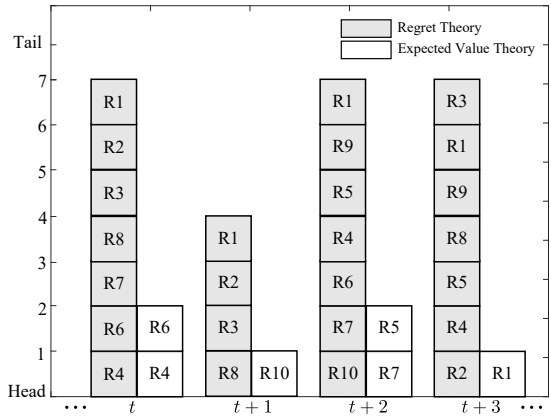}    
\caption{ The comparison between the waiting lines formed according to regret theory and expected value theory, respectively, under the high sensor accuracy low human cost condition.}  
\label{fig:WaitingLine}             
\end{figure}

The performance of the two decision-making models can be further shown in Table~\ref{tb:HighAccLowCost}. 
In the high senor accuracy low human cost condition, robot detection performed well and the operator was fresh thus was of low costs. 
The accuracy of detection, indicated by the percentage of objects found, was much higher for the regret decision model than the expected value decision model. 
The improvement of accuracy was due to the risk-averse attitude, showing by the willingness to accept more human service requests and longer task duration.   
\begin{table}[ht!]
\renewcommand{\arraystretch}{1.3}
\caption{Teaming Performance (High Sensor Accuracy Low Human Cost)}
\label{tb:HighAccLowCost}
\centering
\begin{tabular}{c c c}
\hline
		&Regret Theory & Expected Value  \\
\hline
Average queue length: & 5.0 &		2.3		\\
Percentage of objects found: &  100\%	&		66\% 		\\
Number of human services: &  970 	&		349	\\
Task duration (steps): & 385 &	168			\\
\hline
\end{tabular}
\end{table}

Opposite to \figurename~\ref{fig:WaitingLine} and Table~\ref{tb:HighAccLowCost}, when under the low sensor accuracy high human cost condition, the waiting lines were shorter for the regret decision model, see Table~\ref{tb:LowAccHighCost}.   
In this severe situation, robot detection performed poorly and the operator was very tired thus was costly in providing service. 
Shown by the comparisons in the percentage of objects found and the numbers of human services, the queuing based on the regret decision model became risk-seeking: It did not want to provide service for only a small improvement in accuracy. 
\begin{table}[t!]
\renewcommand{\arraystretch}{1.3}
\caption{Teaming Performance ( Low Sensor Accuracy High Human Cost)}
\label{tb:LowAccHighCost}
\centering
\begin{tabular}{c c c}
\hline
		&Regret Theory & Expected Value \\
\hline
Average queue length: & 0.8 &		1.84	\\
Percentage of objects found: &  15\%	&		19\% 		\\
Number of human services: &  44 	&		134	\\
Task duration (steps): & 	58 &	72			\\
\hline
\end{tabular}
\end{table}

From the human operator's perspective, the decisions made by the regret decision model are more acceptable. 
When the human is fresh, he/she tends to, and is able to, avoid additional cost in the task by providing extra effort.
However, when the human is tired, the ergonomics becomes important. 
Even though the human provided extra effort, the outcomes in the task would only improve marginally. 
A wiser decision thus is to just save effort. 

In \figurename~\ref{fig:PathPlanning}, we show the planned path of one robot using dynamic programming with receding horizon. 
The grid of this region is 20 by 20, because we want to show a longer planning horizon. 
The sensor model used here was  $p(o_a|s_a)=0.7$ and $p(o_p|s_p)=0.9$.
The transition cost was proportional to the distance $d$ between cells and was set as $-2\,d$.   
The gray scale in each cell indicates the expected local cost within the cell. 
As shown, the planned path was a compromise between the local cost and the transition cost: The robot would visit the cell with lowest expected local cost but within its vicinity.  
\begin{figure}[!th]
\centering
\includegraphics[width=7.5cm]{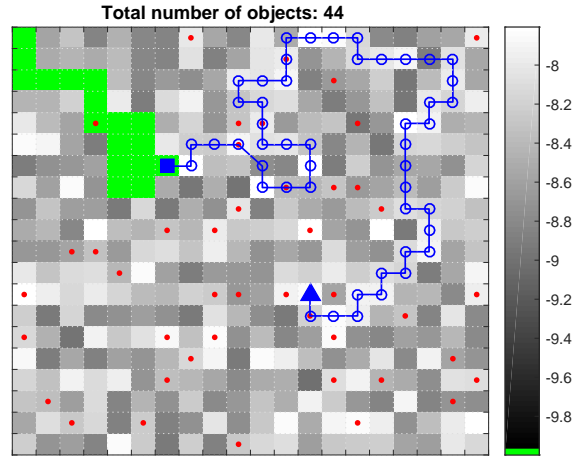}  
\caption{The planned path of one robot in its assigned region is shown. The square and the triangle indicate the current location and the end of the planning horizon, respectively. The green region has been visited. The dots indicate the true locations of the objects that are unknown to the robot. The bar indicates the expected local cost in each cell. The length of the planning horizon is 50.}
\label{fig:PathPlanning}             
\end{figure}

\section{Conclusion}\label{sec:conclusion}
Considering the influence of regret emotion in a robotic decision-making model is important, as it will bring more human-like characteristics, such as risk-seeking and risk-averse. 
We incorporated the decision model based on regret theory into the coordination of a team with one human and multiple robots.
The coordination included robot decision-making, queue-ordering and path planning. 
The simulated results show that the human-like decision model does bring appealing effects to the human-robot teaming. 

\bibliographystyle{IEEEtran}
\bibliography{IEEEabrv,Reference}

\begin{thebibliography}{10}
\providecommand{\url}[1]{#1}
\csname url@samestyle\endcsname
\providecommand{\newblock}{\relax}
\providecommand{\bibinfo}[2]{#2}
\providecommand{\BIBentrySTDinterwordspacing}{\spaceskip=0pt\relax}
\providecommand{\BIBentryALTinterwordstretchfactor}{4}
\providecommand{\BIBentryALTinterwordspacing}{\spaceskip=\fontdimen2\font plus
\BIBentryALTinterwordstretchfactor\fontdimen3\font minus
  \fontdimen4\font\relax}
\providecommand{\BIBforeignlanguage}[2]{{%
\expandafter\ifx\csname l@#1\endcsname\relax
\typeout{** WARNING: IEEEtran.bst: No hyphenation pattern has been}%
\typeout{** loaded for the language `#1'. Using the pattern for}%
\typeout{** the default language instead.}%
\else
\language=\csname l@#1\endcsname
\fi
#2}}
\providecommand{\BIBdecl}{\relax}
\BIBdecl

\bibitem{kolling2016human}
A.~Kolling, P.~Walker, N.~Chakraborty, K.~Sycara, and M.~Lewis, ``Human
  interaction with robot swarms: A survey,'' \emph{IEEE Transactions on
  Human-Machine Systems}, vol.~46, no.~1, pp. 9--26, 2016.

\bibitem{chen2014human}
J.~Y. Chen and M.~J. Barnes, ``Human--agent teaming for multirobot control: A
  review of human factors issues,'' \emph{IEEE Transactions on Human-Machine
  Systems}, vol.~44, no.~1, pp. 13--29, 2014.

\bibitem{parasuraman2000model}
R.~Parasuraman, T.~B. Sheridan, and C.~D. Wickens, ``A model for types and
  levels of human interaction with automation,'' \emph{IEEE Transactions on
  systems, man, and cybernetics-Part A: Systems and Humans}, vol.~30, no.~3,
  pp. 286--297, 2000.

\bibitem{chen2012supervisory}
J.~Y. Chen and M.~J. Barnes, ``Supervisory control of multiple robots: Effects
  of imperfect automation and individual differences,'' \emph{Human Factors},
  vol.~54, no.~2, pp. 157--174, 2012.

\bibitem{kleinman1970optimal}
D.~L. Kleinman, S.~Baron, and W.~Levison, ``An optimal control model of human
  response part i: Theory and validation,'' \emph{Automatica}, vol.~6, no.~3,
  pp. 357--369, 1970.

\bibitem{keeney1993decisions}
R.~L. Keeney and H.~Raiffa, \emph{Decisions with multiple objectives:
  preferences and value trade-offs}.\hskip 1em plus 0.5em minus 0.4em\relax
  Cambridge university press, 1993.

\bibitem{marsella2009ema}
S.~C. Marsella and J.~Gratch, ``Ema: A process model of appraisal dynamics,''
  \emph{Cognitive Systems Research}, vol.~10, no.~1, pp. 70--90, 2009.

\bibitem{velasquez1998robots}
J.~D. Vel{\'a}squez, ``When robots weep: emotional memories and
  decision-making,'' in \emph{AAAI/IAAI}, 1998, pp. 70--75.

\bibitem{loomes1982regret}
G.~Loomes and R.~Sugden, ``Regret theory: An alternative theory of rational
  choice under uncertainty,'' \emph{The economic journal}, vol.~92, no. 368,
  pp. 805--824, 1982.

\bibitem{kahneman1979prospect}
D.~Kahneman and A.~Tversky, ``Prospect theory: An analysis of decision under
  risk,'' \emph{Econometrica}, vol.~47, no.~2, pp. 263--292, 1979.

\bibitem{camille2004involvement}
N.~Camille, G.~Coricelli, J.~Sallet, P.~Pradat-Diehl, J.-R. Duhamel, and
  A.~Sirigu, ``The involvement of the orbitofrontal cortex in the experience of
  regret,'' \emph{Science}, vol. 304, no. 5674, pp. 1167--1170, 2004.

\bibitem{liao2017quantitative}
Z.~Liao, L.~Jiang, and Y.~Wang, ``A quantitative measure of regret in
  decision-making for human-robot collaborative search tasks,'' in
  \emph{American Control Conference (ACC), 2017}.\hskip 1em plus 0.5em minus
  0.4em\relax IEEE, 2017, pp. 1524--1529.

\bibitem{kontek2017range}
K.~Kontek and M.~Lewandowski, ``Range-dependent utility,'' \emph{Management
  Science}, vol.~64, no.~6, pp. 2812--2832, 2017.

\bibitem{quiggin1982theory}
J.~Quiggin, ``A theory of anticipated utility,'' \emph{Journal of Economic
  Behavior \& Organization}, vol.~3, no.~4, pp. 323--343, 1982.

\bibitem{jiang2019human}
L.~Jiang and Y.~Wang, ``A human-computer interface design for quantitative
  measure of regret theory,'' \emph{IFAC-PapersOnLine}, vol.~51, no.~34, pp.
  15--20, 2019.

\bibitem{prelec1998probability}
D.~Prelec, ``The probability weighting function,'' \emph{Econometrica}, pp.
  497--527, 1998.

\end{thebibliography}

\end{document}